\documentclass[letterpaper, 10 pt, conference]{ieeeconf}  

\IEEEoverridecommandlockouts                              

\usepackage{graphics} 
\graphicspath{ {./images/} }

\usepackage{epsfig} 

\usepackage{multirow}

\usepackage{amsmath} 
\usepackage{amssymb}  
\usepackage{ulem}
\usepackage{nccmath}
\usepackage{mathrsfs}
 
\usepackage[utf8]{inputenc}
\usepackage[english]{babel}
\newtheorem{theorem}{\noindent\hspace{1em}\bf Theorem}

\usepackage{algorithm}
\usepackage{algpseudocode}

\usepackage{kotex}

\usepackage{xcolor}

\usepackage{tikz}

\title{\LARGE \bf
Safety-Critical Control for Aerial Physical Interaction in Uncertain Environment}

\author{Jeonghyun Byun$^{1}$, Yeonjoon Kim$^{1}$, Dongjae Lee$^{1}$, and H. Jin Kim$^{1}$
\thanks{$^{1}$ The authors are with the Department of Aerospace Engineering, Automation and System Research Institute(ASRI) and Institute of Advanced Aerospace Technology(IAAT), Seoul National University, Seoul, South Korea.
        {\tt\small \{quswjdgus97, 0831joon, ehdwo713, hjinkim\}@snu.ac.kr}}
}

\begin{document}

\maketitle
\thispagestyle{empty}
\pagestyle{empty}

\begin{abstract}


Aerial manipulation for safe physical interaction with their environments is gaining significant momentum in robotics research. 
In this paper, we present a disturbance-observer-based safety-critical control for a fully actuated aerial manipulator interacting with both static and dynamic structures.
Our approach centers on a safety filter that dynamically adjusts the desired trajectory of the vehicle's pose, accounting for the aerial manipulator's dynamics, the disturbance observer's structure, and motor thrust limits.
We provide rigorous proof that the proposed safety filter ensures the forward invariance of the safety set—representing motor thrust limits—even in the presence of disturbance estimation errors.
To demonstrate the superiority of our method over existing control strategies for aerial physical interaction, we perform comparative experiments involving complex tasks, such as pushing against a static structure and pulling a plug firmly attached to an electric socket. 
Furthermore, to highlight its {repeatability} in scenarios with sudden dynamic changes, {we perform repeated tests of pushing a movable cart and extracting a plug from a socket.} 
These experiments confirm that our method not only outperforms existing methods but also excels in handling tasks with rapid dynamic variations.

\end{abstract}

\section{INTRODUCTION}

Aerial manipulators have gained popularity in robotics research due to their ability to combine the maneuverability of unmanned aerial vehicles (UAVs) with the versatility of robotic manipulators. Their capability to physically interact with the environment makes them ideal for a wide range of applications involving physical interaction such as drawer-opening \cite{kim2015operating}, door-opening \cite{lee2020aerial}, plug-pulling \cite{byun2021stability}, window-cleaning \cite{sun2021switchable}, multi-manual object manipulation \cite{shahriari2022passivity}, non-destructive testing (NDT) \cite{tong2024passive} and heavy object pushing \cite{hwang2024autonomous}.

Such tasks, often referred to as aerial physical interaction (APhI) depicted in Fig. \ref{fig: thumbnail}, require consideration of two important factors when designing controllers. The first is the external disturbance generated by the physical interaction between the vehicle and the surrounding environment, and the second is the vehicle's actuation limit such as motor thrust limit. However, very few studies have focused on controller design for APhI that simultaneously accounts for both disturbance attenuation and actuator limitations.

\begin{figure}[t]
\centering
\vspace{0.25cm}
\includegraphics[width = 0.45\textwidth]{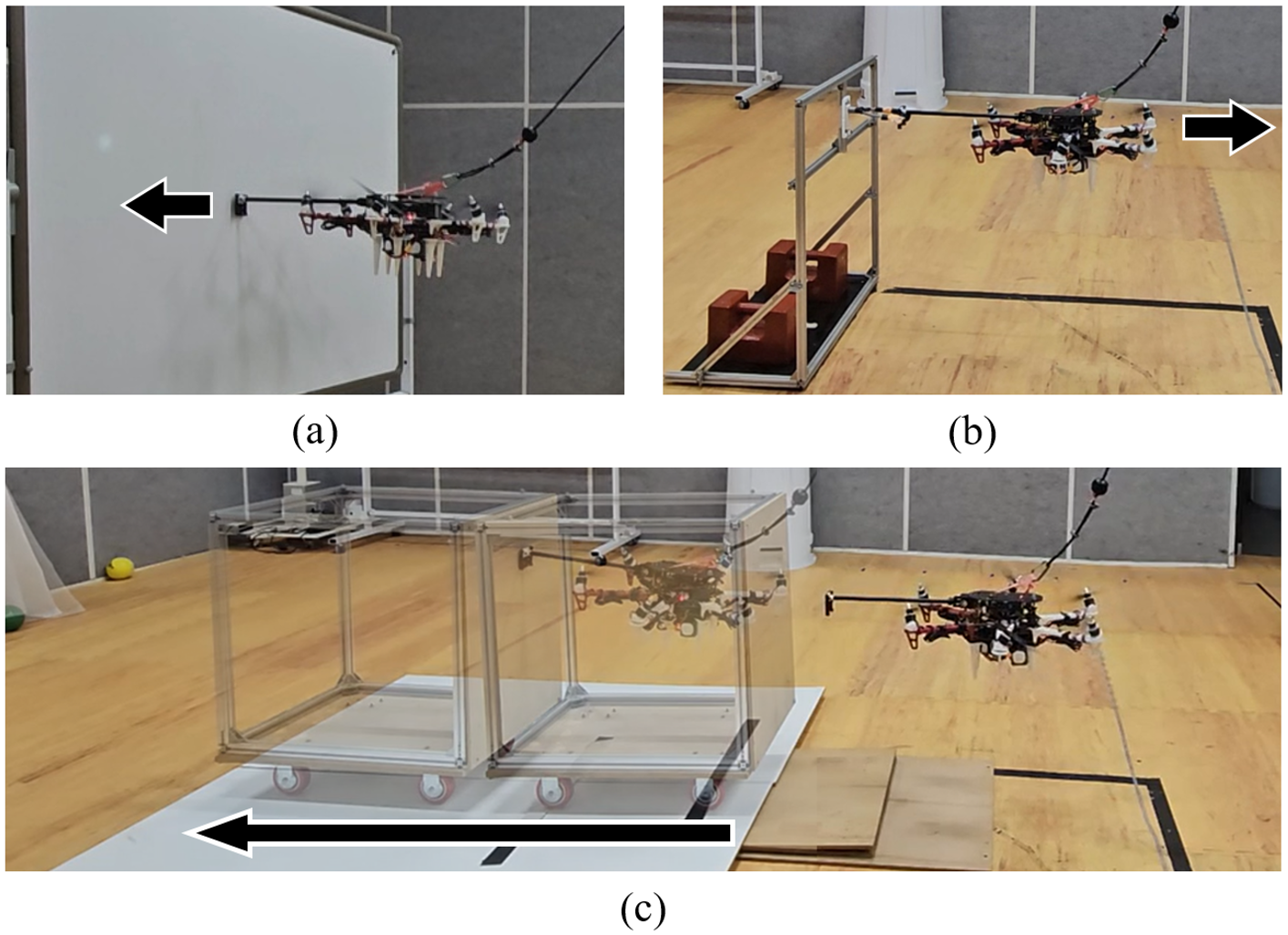}
\vspace{-0.35cm}
\caption{ Aerial physical interaction (APhI) experiments conducted for the validation of the proposed method. (a) Pushing a static structure. (b) Pulling an object from a static structure. (c) Pushing a movable object.} \label{fig: thumbnail}
\end{figure}

In \cite{convens2017control}, the control of a fully actuated UAV with actuator saturation was introduced, and in \cite{afifi2022toward}, the authors presented an actuator limit-aware control strategy for an aerial manipulator physically interacting with a human operator. However, they did not consider the effect of external disturbances or model uncertainties when designing their controllers and only conducted numerical simulations. 

In \cite{brunner2022energy} and \cite{benzi2022adaptive}, energy tank-based controllers for aerial manipulators interacting with uncertain dynamic environments were presented. They proved the vehicle's safety under the effect of external disturbances or model uncertainties. Also, \cite{cuniato2022power} proposed a power-based safety layer for aerial vehicles physically interacting with surrounding environments, and \cite{chaikalis2024optimization} introduced a control method for human-robot APhI that considers motor thrust limitations and external disturbances. However, these four controllers assume that the interaction wrench can be measured by a force/torque sensor or estimated based on a derived model. This reliance may not be suitable for aerial manipulators not equipped with force/torque operating in windy environments.

\cite{lee2024saturated} presented a saturated robust controller for a fully actuated multirotor taking into account the disturbance rejection and rotor thrust saturation. However, they did not consider the situation that reaching the desired pose of the vehicle is not dynamically feasible, e.g., the magnitude of an external disturbance exceeds the motor thrust limit, and only conducted numerical simulations.

In this paper, we present a disturbance-observer (DOB)-based safety-critical control of a fully actuated aerial manipulator physically interacting with static or movable structures. 
To that end, we derive the dynamic model of the aerial manipulator configured with a fully actuated hexacopter and a rigidly attached robotic arm. 
Then, the control framework including a DOB-based control law and a safety filter is introduced. 
While the DOB-based control law is designed as presented in \cite{ha2018disturbance}, for the safety filter, we develop a system model that combines the aerial manipulator's model and DOB structure and formulate an optimization problem to determine the desired pose and twist of the vehicle that conforms to the constraints on motor thrust. In particular, our main contributions are arranged as follows:
\begin{itemize}
    \item We design a safety-critical controller that can conduct various types of APhI with uncertain environments in the absence of interaction wrench measurement or estimation.
    \item We show that the safety set representing the motor thrust limits is forward invariant under the proposed safety filter. 
    \item We validate that our controller outperforms existing approaches for APhI control through pushing and pulling experiments with both static and dynamic structures.
\end{itemize}

This paper is outlined as follows; In Section II, we formulate the dynamic model of our fully actuated aerial manipulator, and the control framework that includes DOB-based control law and safety filter is introduced in Section III. In Section IV, we present a theoretical analysis on the safety filter, and Section V covers the experimental validation of the proposed method.

\textbf{Notations:} $\boldsymbol{0_{i \times j}}$, $\boldsymbol{I_{i}}$, $\boldsymbol{e_3}$ and $\boldsymbol{E_i} \in {\mathbb{R}}^6$ represent the $i\times j$ zero matrix, $i \times i$ identity matrix, $[0;0;1]$ and ${\mathbb{R}}^6$ vector where its $i^{\textrm{th}}$ element is one and the other elements are zero, respectively. For scalars $a_1,\cdots,a_N$, we let $c a_1$ and $s a_1$ denote $\cos{a_1}$ and $\sin{a_1}$, respectively, and diag$\{a_1,\cdots,a_N\}$ represents an $N \times N$ diagonal matrix where the $(i,i)^{\textrm{th}}$ element is $a_i$ and and all off-diagonal elements are zero. For vectors $\boldsymbol{\alpha}$ and $\boldsymbol{\beta}$, we let $\alpha_i$ denote the $i$-th element of $\alpha$, and if $\boldsymbol{\alpha} \in {\mathbb{R}}^{3}$ and $\boldsymbol{\beta} \in {\mathbb{R}}^{3}$, $[\boldsymbol{\alpha}]_{\times} \in \mathbb{R}^{3\times3}$ means the operator which maps $\boldsymbol{\alpha}$ into a skew-symmetric matrix such as $[\boldsymbol{\alpha}]_{\times}\boldsymbol{\beta} = \boldsymbol{\alpha} \times \boldsymbol{\beta}$. Also, for matrices $\boldsymbol{A_1}, \cdots, \boldsymbol{A_N}$, $\textrm{blkdiag}\{ \boldsymbol{A_1}, \cdots, \boldsymbol{A_N} \}$ represents a block diagonal matrix obtained by aligning $\boldsymbol{A_1},\cdots,,\boldsymbol{A_N}$. Moreover, we abbreviate the phrase "with respect to" to w.r.t.. 

  \section{Modelling}

\begin{figure}[t]
\centering
\vspace{0.25cm}
\includegraphics[width = 0.33\textwidth]{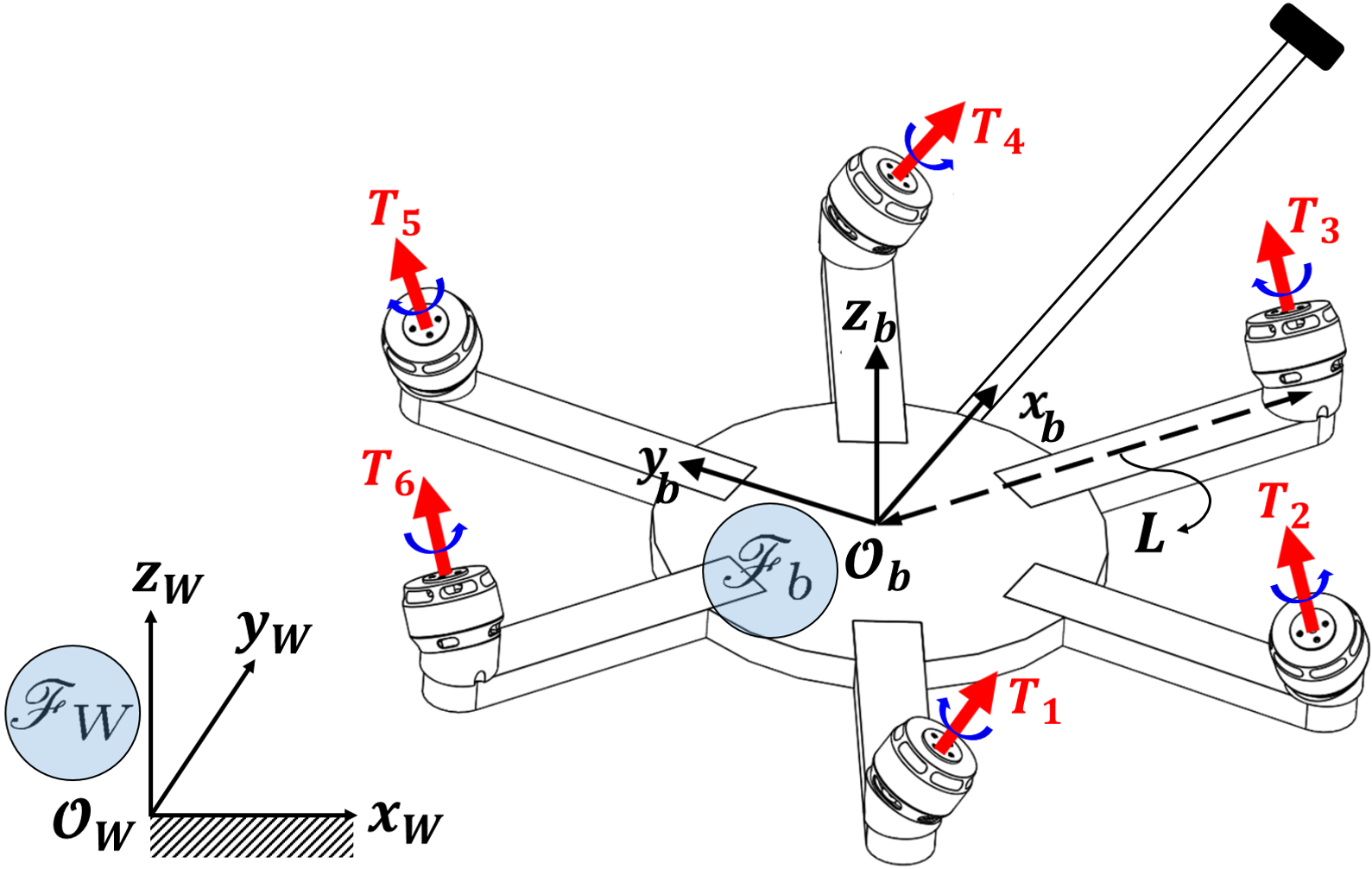}
\vspace{-0.20cm}
\caption{An aerial manipulator configured with a fully actuated multirotor and a rigidly attached robotic arm. The thrust values $T_1, \cdots, T_6$ configure the thrust vector $\boldsymbol{T} \in {\mathbb{R}}^{6}$, and $\mathscr{F}_{W}$ and $\mathscr{F}_b$ represent the Earth-fixed and multirotor frames, respectively.} \label{fig: aerial manipulator with fully actuated multirotor}
\end{figure}

Fig. \ref{fig: aerial manipulator with fully actuated multirotor} illustrates an aerial manipulator configured with a fully actuated multirotor with six tilted motors and a rigidly attached robotic arm. Also, two key coordinate frames, the Earth-fixed frame $\mathscr{F}_W$ and the multirotor frame $\mathscr{F}_b$, are displayed.  

The generalized coordinate of the aerial manipulator, $\boldsymbol{q} \triangleq [\boldsymbol{p};\boldsymbol{\phi}]$, consists of the position of $\mathscr{F}_b$ w.r.t. $\mathscr{F}_W$ expressed in $\mathscr{F}_W$, $\boldsymbol{p} \in {\mathbb{R}}^{3}$, and ZYX Euler angles, $\boldsymbol{\phi} \in {\mathbb{R}}^{3}$. Then, the Euler-Lagrange model of the aerial manipulator is formulated as follows \cite{lee2024saturated}: 
\begin{equation} \label{eqn: dynamic model of the aerial manipulator}
    \begin{split}
       \boldsymbol{M}(\boldsymbol{\phi})\boldsymbol{\ddot{q}} + \boldsymbol{C}(\boldsymbol{\phi},\boldsymbol{\dot{\phi}}) + \boldsymbol{G}  =  \boldsymbol{\tau} + \boldsymbol{\tau_{ext}}
    \end{split}
\end{equation}
where $\boldsymbol{\tau} \in {\mathbb{R}}^6$ and $\boldsymbol{\tau_{ext}} \in {\mathbb{R}}^6$ represent the generalized control wrench and external disturbance, respectively, and 
\begin{equation*}
    \begin{split}
        \boldsymbol{M}(\boldsymbol{\phi}) \triangleq& \begin{bmatrix}
            m\boldsymbol{I_3} & \boldsymbol{0_{3\times3}} \\
            \boldsymbol{0_{3\times3}} & \boldsymbol{Q^{\top}}\boldsymbol{J}\boldsymbol{Q}
        \end{bmatrix}, \\
        \boldsymbol{C} (\boldsymbol{\phi},\boldsymbol{\dot{\phi}}) \triangleq& \begin{bmatrix}
            \boldsymbol{0_{3\times1}} \\
            \boldsymbol{Q^{\top}}(\boldsymbol{J}\boldsymbol{\dot{Q}}\boldsymbol{\dot{\phi}} + [\boldsymbol{Q}\boldsymbol{\dot{\phi}}]_{\times}\boldsymbol{J}\boldsymbol{Q}\boldsymbol{\dot{\phi}})
        \end{bmatrix}, \ \boldsymbol{G} \triangleq \begin{bmatrix}
            mg\boldsymbol{e_3} \\ 
            \boldsymbol{0_{3\times1}}
        \end{bmatrix}
    \end{split}
\end{equation*}
with the mass and moment of inertia of the aerial manipulator, $m$ and $\boldsymbol{J} \in {\mathbb{R}}^{3\times3}$, and the gravitational acceleration, $g$. Also, we let $\boldsymbol{Q} \in {\mathbb{R}}^{3\times3}$ denote the mapping matrix satisfying $\boldsymbol{\omega} = \boldsymbol{Q}\boldsymbol{\dot{\phi}}$ where $\boldsymbol{\omega} \in {\mathbb{R}}^{3}$ represents the angular velocity of the vehicle w.r.t. $\mathscr{F}_{W}$ expressed in $\mathscr{F}_{b}$. Meanwhile, the relationship between $\boldsymbol{\tau}$ and $\boldsymbol{T} \triangleq [T_1;\cdots;T_6]$ is derived as follows:
\begin{equation} \label{eqn: relation between tau and T}
    \begin{split}
        \boldsymbol{\tau} = \boldsymbol{B}(\boldsymbol{\phi})\boldsymbol{\Xi}\boldsymbol{T}
    \end{split}
\end{equation}
where $\boldsymbol{B}(\boldsymbol{\phi}) \triangleq \textrm{blkdiag}\{\boldsymbol{R}, \boldsymbol{Q^{\top}}\}$ and 
\begin{equation} \label{eqn: definition of Xi}
    \begin{split}
        \boldsymbol{\Xi} \triangleq& \begin{bmatrix}
            \tfrac{1}{2}s\alpha & -s\alpha & \tfrac{1}{2}s\alpha & \tfrac{1}{2}s\alpha & -s\alpha & \tfrac{1}{2}s\alpha \\ 
             -\tfrac{\sqrt{3}}{2}s\alpha & 0 & \tfrac{\sqrt{3}}{2}s\alpha &  -\tfrac{\sqrt{3}}{2}s\alpha & 0 & \tfrac{\sqrt{3}}{2}s\alpha \\ 
             c\alpha & c\alpha & c\alpha & c\alpha & c\alpha & c\alpha \\ 
               -\tfrac{1}{2}P_1 & -P_1 & -\tfrac{1}{2}P_1 & \tfrac{1}{2}P_1 & P_1 & \tfrac{1}{2}P_1 \\ 
               \tfrac{\sqrt{3}}{2}P_1 & 0 & -\tfrac{\sqrt{3}}{2}P_1 & -\tfrac{\sqrt{3}}{2}P_1 & 0 & \tfrac{\sqrt{3}}{2}P_1 \\ 
               P_2 & -P_2 & P_2 & -P_2 & P_2 & -P_2
        \end{bmatrix}.
    \end{split}
\end{equation}
with $P_1 \triangleq Lc\alpha + k_fs\alpha$ and $P_2 \triangleq Ls\alpha - k_fc\alpha$. In (\ref{eqn: definition of Xi}), $L$, $\alpha$ and $k_f$ represent the length from the vehicle's origin to each propeller, fixed tilt angle of each motor and thrust-to-torque coefficient, respectively. 

To address the effect of model uncertainties, we rearrange (\ref{eqn: dynamic model of the aerial manipulator}) as follows:
\begin{equation} \label{eqn: rearranged dynamic model of the aerial manipulator}
    \begin{split}
       \boldsymbol{\hat{M}}(\boldsymbol{\phi})\boldsymbol{\ddot{q}} + \boldsymbol{\hat{C}}(\boldsymbol{\phi},\boldsymbol{\dot{\phi}}) + \boldsymbol{\hat{G}} = \boldsymbol{\tau} + \boldsymbol{d}
    \end{split}
\end{equation}
where $\boldsymbol{\hat{M}}(\boldsymbol{\phi})$, $\boldsymbol{\hat{C}}(\boldsymbol{\phi},\boldsymbol{\dot{\phi}})$ and $\boldsymbol{\hat{G}}$ are the nominal values of $\boldsymbol{M}(\boldsymbol{\phi})$, $\boldsymbol{C}(\boldsymbol{\phi},\boldsymbol{\dot{\phi}})$ and $\boldsymbol{G}$, respectively, and the lumped disturbance $\boldsymbol{d} \in {\mathbb{R}}^6$ satisfies: 
\begin{multline*}
    \boldsymbol{d} \triangleq (\boldsymbol{\hat{M}}(\boldsymbol{\phi})-\boldsymbol{M}(\boldsymbol{\phi}))\boldsymbol{\ddot{q}} + \boldsymbol{\hat{C}}(\boldsymbol{\phi},\boldsymbol{\dot{\phi}}) - \boldsymbol{C}(\boldsymbol{\phi},\boldsymbol{\dot{\phi}}) \\+ \boldsymbol{\hat{G}} - \boldsymbol{G} + \boldsymbol{\tau_{ext}}.
\end{multline*}

\section{Control Framework} \label{section: control framework}

A DOB-based controller introduced in \cite{ha2018disturbance} makes the state of a robotic system follow its desired trajectory within a small bound for all $t \geq t_0$ even under the effect of time-varying external disturbances and model uncertainties. Hence, we adopt this control structure to conduct APhI. 
However, if it is not dynamically feasible to reach the vehicle's target pose, e.g., the target position is inside the wall while exerting on a static wall, the aerial system may become unstable since one of the motors cannot generate the commanded thrust value due to its limit. 
Thus, we need a safety filter that modifies the desired pose trajectory based on the system model (\ref{eqn: rearranged dynamic model of the aerial manipulator}), the DOB's structure and the feasible range of the motor thrust.

\begin{figure}[t]
\centering
\vspace{0.25cm}
\includegraphics[width = 0.35\textwidth]{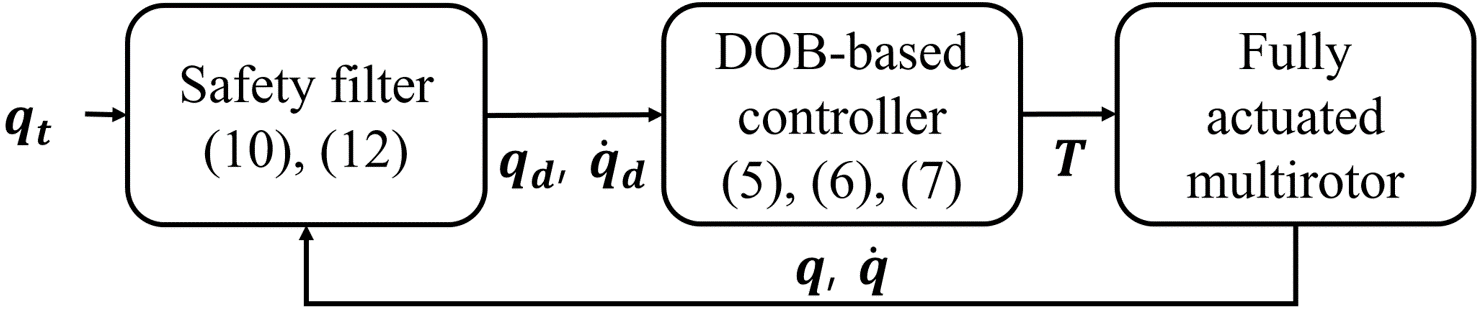}
\vspace{-0.33cm}
\caption{Overall controller diagram for safe APhI} \label{fig: overall controller diagram for safe APhI}
\end{figure}

Fig. \ref{fig: overall controller diagram for safe APhI} presents the overall controller diagram for safe APhI. Based on the target pose set by an external planner (e.g., joystick or ground computer station), $\boldsymbol{q_{t}} \in {\mathbb{R}}^6$, the desired pose and twist, $\boldsymbol{q_d} \in {\mathbb{R}}^{6}$ and $\boldsymbol{\dot{q}_d} \in {\mathbb{R}}^{6}$, are calculated by the safety filter. Then, the DOB-based control law generates $\boldsymbol{T}$.

\subsection{DOB-Based Control Law}

According to \cite{ha2018disturbance}, the estimation of the lumped disturbance, $\boldsymbol{\hat{d}} \in {\mathbb{R}}^6$, is formulated as follows: 
\begin{equation} \label{eqn: disturbance observer}
    \begin{split}
        \boldsymbol{\hat{d}} =& -\boldsymbol{\hat{M}}(\boldsymbol{\phi})(\boldsymbol{\mu^{-1}}\boldsymbol{\Gamma_{\zeta}}(\boldsymbol{\zeta} - \boldsymbol{\dot{q}}) + \boldsymbol{\chi}) + \boldsymbol{\hat{C}}(\boldsymbol{\phi},\boldsymbol{\dot{\phi}}) + \boldsymbol{\hat{G}} \\
        \boldsymbol{\dot{\zeta}} =& -\boldsymbol{\mu^{-1}}\boldsymbol{\Gamma_{\zeta}}(\boldsymbol{\zeta} - \boldsymbol{\dot{q}}) \\
        \boldsymbol{\dot{\chi}} =& -\boldsymbol{\mu^{-1}}\boldsymbol{\Gamma_{\chi}}(\boldsymbol{\chi} - \boldsymbol{\hat{M}^{-1}}(\boldsymbol{\phi})\boldsymbol{\tau})
    \end{split}
\end{equation}
where $\boldsymbol{\zeta} \in {\mathbb{R}}^{6}$ and $\boldsymbol{\chi} \in {\mathbb{R}}^{6}$ represent the filtered values of $\boldsymbol{\dot{q}}$ and $\boldsymbol{\hat{M}^{-1}}(\boldsymbol{\phi})\boldsymbol{\tau}$, respectively, with constant positive definite matrices $\boldsymbol{\Gamma_{\zeta}}, \boldsymbol{\Gamma_{\chi}} \in {\mathbb{R}}^{6\times6}_{>0}$ and a matrix $\boldsymbol{\mu} \triangleq \textrm{diag}\{\mu_1,\cdots,\mu_6\}$ configured with positive constants lower than one. Then, based on $\boldsymbol{q_d}$ and $\boldsymbol{\dot{q}_d}$ acquired from the safety filter, the control law for $\boldsymbol{\tau}$ is designed as follows:
\begin{multline} \label{eqn: DOB-based control law}
        \boldsymbol{\tau}(\boldsymbol{x}) = \boldsymbol{\hat{M}}(\boldsymbol{\phi})(\boldsymbol{K_d}\boldsymbol{\dot{e}} + \boldsymbol{K_p}\boldsymbol{e}) + \boldsymbol{\hat{C}}(\boldsymbol{\phi},\boldsymbol{\dot{\phi}}) + \boldsymbol{\hat{G}} - \boldsymbol{\hat{d}}
\end{multline}
where $\boldsymbol{e} \triangleq \boldsymbol{q_d} - \boldsymbol{q}$ and $\boldsymbol{x} \triangleq [\boldsymbol{q};\boldsymbol{\dot{q}};\boldsymbol{\zeta};\boldsymbol{\chi};\boldsymbol{q_d};\boldsymbol{\dot{q}_d}]$ with controller gains $\boldsymbol{K_p} \in {\mathbb{R}}^{6\times6}_{>0}$ and $\boldsymbol{K_d} \in {\mathbb{R}}^{6\times6}_{>0}$. By (\ref{eqn: relation between tau and T}), $\boldsymbol{T}(\boldsymbol{x})$ is calculated as follows:
\begin{equation} \label{eqn: conversion from control wrench to motor thrust vector}
    \begin{split}
        \boldsymbol{T}(\boldsymbol{x}) =& \boldsymbol{\Xi^{-1}}\boldsymbol{B^{-1}}(\boldsymbol{\phi})\boldsymbol{\tau}(\boldsymbol{x}).
    \end{split}
\end{equation}

\subsection{Safety Filter}

In the safety filter, $\boldsymbol{q_d}$ and $\boldsymbol{\dot{q}_d}$ is adjusted based on (\ref{eqn: rearranged dynamic model of the aerial manipulator}) - (\ref{eqn: DOB-based control law}) and the limit of motor thrust. To this end, we need to formulate a control-affine system with $\boldsymbol{x}$ as a state and the vehicle's acclereation $\boldsymbol{\ddot{q}_d} \in {\mathbb{R}}^{6}$ as an input.

By substituting (\ref{eqn: disturbance observer}) and (\ref{eqn: DOB-based control law}) for (\ref{eqn: rearranged dynamic model of the aerial manipulator}), we obtain the system w.r.t. $\boldsymbol{x}$ as follows:
\begin{equation} \label{eqn: general system formulation}
    \begin{split}
        \boldsymbol{\dot{x}} =& \boldsymbol{f}(\boldsymbol{x}) + \boldsymbol{g}\boldsymbol{\ddot{q}_d} + \boldsymbol{\rho}(\boldsymbol{x},\boldsymbol{\tilde{d}})
    \end{split}
\end{equation}
where $\boldsymbol{g} $$ = $$ [\boldsymbol{0_{30\times6}} ; \boldsymbol{I_6}]$,  $\boldsymbol{\rho}(\boldsymbol{x},\boldsymbol{\tilde{d}})$$ =$$ [\boldsymbol{0_{6 \times 1}};-\boldsymbol{\hat{M}^{-1}}(\boldsymbol{\phi})\boldsymbol{\tilde{d}};\boldsymbol{0_{24 \times 1}}]$ and
\begin{equation*}
    \begin{split}
        \boldsymbol{f}(\boldsymbol{x}) =& \begin{bmatrix}
            \boldsymbol{\dot{q}} \\
            \boldsymbol{K_d}\boldsymbol{\dot{e}} + \boldsymbol{K_p}\boldsymbol{e} \\
            -\boldsymbol{\mu^{-1}}\boldsymbol{\Gamma_{\zeta}}(\boldsymbol{\zeta} - \boldsymbol{\dot{q}}) \\
            \boldsymbol{\mu^{-1}}\boldsymbol{\Gamma_{\chi}}(\boldsymbol{K_d}\boldsymbol{\dot{e}} + \boldsymbol{K_p}\boldsymbol{e} + \boldsymbol{\mu^{-1}}\boldsymbol{\Gamma_{\zeta}}(\boldsymbol{\zeta} - \boldsymbol{\dot{q}})) \\
            \boldsymbol{\dot{q}_d} \\
            \boldsymbol{0_{6\times1}}
        \end{bmatrix}.
    \end{split}
\end{equation*}

Our main objective is to make $\boldsymbol{x}$ stay inside its safety set constructed as:
\begin{equation*}
    \mathcal{C}_T \triangleq \{\boldsymbol{x} \in {\mathbb{R}}^{36} \ | \ T_m \leq T_{i}(\boldsymbol{x}) \leq T_M, \ i = 1,\cdots,6 \}
\end{equation*}
where $T_m$ and $T_M$ are the minimum and maximum values of each motor's thrust. Hence, we first define six barrier functions $h_{T,1}(\boldsymbol{x}), \cdots, h_{T,6}(\boldsymbol{x})$ as follows:
\begin{equation} \label{eqn: barrier functions for the input saturation}
    \begin{split}
        h_{T,i}(\boldsymbol{x}) \triangleq& (\tfrac{T_M - T_m}{2})^2 - (T_i(\boldsymbol{x}) - \tfrac{T_M + T_m}{2})^2
    \end{split}
\end{equation}
so that $\mathcal{C}_T = \mathcal{C}_{T,1} \cap \cdots \cap \mathcal{C}_{T,6}$ with $\mathcal{C}_{T,i} \triangleq \{\boldsymbol{x} \in {\mathbb{R}}^{36} \ | \ h_{T,i}(\boldsymbol{x}) \geq 0 \}$. To ensure that $T_i(\boldsymbol{x})$ does not violate its limit for all $t \geq t_0$, the desired acceleration of the vehicle, $\boldsymbol{\ddot{q}_d}$, has to satisfy the following inequality:
\begin{multline*} 
    0 \leq \dot{h}_{T,i}(\boldsymbol{x},\boldsymbol{\dot{x}}) + \gamma_{T,i}h_{T,i}(\boldsymbol{x}), \ {}^{\forall} (x, i) \in \mathcal{D} \times \{1,\cdots,6\}
\end{multline*}
where $\gamma_{T,i}$ is a user-defined constant parameter and $\mathcal{D} \subset {\mathbb{R}}^{36}$ means the domain of $\boldsymbol{x}$. 

\subsubsection{Compensation for DOB Error}

According to (\ref{eqn: general system formulation}), $\dot{h}_{T,i}(\boldsymbol{x},\boldsymbol{\dot{x}})$ is calculated as $\mathcal{L}_{\boldsymbol{f}}h_{T,i}(\boldsymbol{x}) + \mathcal{L}_{\boldsymbol{g}}h_{T_i}(\boldsymbol{x})\boldsymbol{\ddot{q}_d} + \tfrac{\partial h_{T,i} (\boldsymbol{x})}{\partial \boldsymbol{x}}\boldsymbol{\rho}(\boldsymbol{x},\boldsymbol{\tilde{d}})$ and consists of an unknown term $\beta_{T,i}(\boldsymbol{x},\boldsymbol{\tilde{d}}) \triangleq \tfrac{\partial h_{T,i} (\boldsymbol{x})}{\partial \boldsymbol{x}}\boldsymbol{\rho}(\boldsymbol{x},\boldsymbol{\tilde{d}})$. Therefore, we estimate this term by utilizing the method introduced in \cite{alan2022disturbance} as follows:
\begin{equation} \label{eqn: unknown CBF estimation}
    \begin{split}
        \hat{\beta}_{T,i}(\boldsymbol{x},\xi_{T,i}) = k_{\beta,T,i}h_{T,i}(\boldsymbol{x}) - \xi_{T,i}   \\
        \dot{\xi}_{T,i} = k_{\beta,T,i}\{\mathcal{L}_{\boldsymbol{f}}h_{T,i}(\boldsymbol{x}) + \mathcal{L}_{\boldsymbol{g}}h_{T,i}(\boldsymbol{x})\boldsymbol{\ddot{q}_d} + \hat{\beta}_{T,i}(\boldsymbol{x},\xi_{T,i})\}
    \end{split}
\end{equation}
with a positive constant parameter $k_{\beta,T,i} > \gamma_{T,i}$. {According to \cite{alan2022disturbance}, (\ref{eqn: unknown CBF estimation}) guarantees the following inequality:
\begin{equation} \label{eqn: equation on e_T}
    |e_{T,i}(t)| \leq (|e_{T,i}(t_0)|-\tfrac{\beta_{T,h,i}}{k_{\beta,T,i}})e^{-k_{\beta,T,i}(t - t_0)} + \tfrac{\beta_{T,h,i}}{k_{\beta,T,i}}
\end{equation}
where $e_{T,i} \triangleq \beta_{T,i}(\boldsymbol{x},\boldsymbol{\tilde{d}}) - \hat{\beta}_{T,i}(\boldsymbol{x},\xi_{T,i})$ and $\beta_{T,h,i} \geq |\tfrac{d}{dt}\beta_{T,i}(\boldsymbol{x},\boldsymbol{\tilde{d}})|$ } 

\subsubsection{Quadratic Programming (QP) Problem on Solving Desired Acceleration}

Since there exists a difference between $\beta_{T,i}(\boldsymbol{x},\boldsymbol{\tilde{d}})$ and $\hat{\beta}_{T,i}(\boldsymbol{x},\xi_{T,i})$, we define a more conservative inequality on $\boldsymbol{\ddot{q}_d}$ as follows: 
\begin{multline*}
    \sigma_{T,i} \leq \mathcal{L}_{\boldsymbol{f}}h_{T,i}(\boldsymbol{x}) + \mathcal{L}_{\boldsymbol{g}}h_{T_i}(\boldsymbol{x})\boldsymbol{\ddot{q}_d} \\+ \hat{\beta}_{T,i}(\boldsymbol{x},\xi_{T,i}) + \gamma_{T,i}h_{T,i}(\boldsymbol{x})
\end{multline*}
with the positive scalars $\sigma_{T,1}, \cdots, \sigma_{T,6}$.According to \cite{alan2022disturbance}, if $\sigma_{T,i} \geq \textrm{max}\{|e_{T,i}(t_0)|,\tfrac{\beta_{T,h,i}}{k_{\beta,T,i}}\}$, then $0 \leq \dot{h}_{T,i}(\boldsymbol{x},\boldsymbol{\dot{x}}) + \gamma_{T,i}h_{T,i}(\boldsymbol{x})$. Also, since the boundedness of $\boldsymbol{\tilde{d}}$ and $\boldsymbol{\dot{\tilde{d}}}$ is proven in \cite{ha2018disturbance} and $|e_{T,i}(t_0)|$ is also bounded, we can find a large number $\sigma_{T,i}$ such that $\sigma_{T,i} \geq \textrm{max}\{|e_{T,i}(t_0)|,\tfrac{\beta_{T,h,i}}{k_{\beta,T,i}}\}, \ {}^{\forall} i \in \{1,\cdots,6\}$. 

Accordingly, the quadratic programming (QP) problem with the decision variable $\boldsymbol{\ddot{q}_d}$ is formulated as follows:
\begin{equation} \label{eqn: quadratic programming}
    \begin{split}
        \underset{\boldsymbol{\ddot{q}_d}}{\textrm{min}}  \|\boldsymbol{\ddot{q}_d} - \boldsymbol{\ddot{q}_{t}}\|^2 \textrm{ s.t. } \boldsymbol{A_T}\boldsymbol{\ddot{q}_d} + \boldsymbol{\sigma_T} \leq \boldsymbol{\hat{b}_T}
    \end{split}
\end{equation}
where $\boldsymbol{A_T} \triangleq -[\mathcal{L}_{\boldsymbol{g}}h_{T,1}(\boldsymbol{x});\mathcal{L}_{\boldsymbol{g}}h_{T,2}(\boldsymbol{x});\cdots; \mathcal{L}_{\boldsymbol{g}}h_{T,6}(\boldsymbol{x})]$, $\boldsymbol{\sigma_T} \triangleq [\sigma_{T,1};\cdots;\sigma_{T,6}]$ and 
\begin{equation*}
    \begin{split}
        \boldsymbol{\hat{b}_T} \triangleq& \begin{bmatrix}
            \gamma_{T,1}(h(\boldsymbol{x})) + \mathcal{L}_{\boldsymbol{f}}h_{T,1}(\boldsymbol{x}) + \hat{\beta}_{T,1}(\boldsymbol{x},\xi_{T,1}) \\
            \vdots \\
            \gamma_{T,6}(h(\boldsymbol{x})) + \mathcal{L}_{\boldsymbol{f}}h_{T,6}(\boldsymbol{x}) + \hat{\beta}_{T,6}(\boldsymbol{x},\xi_{T,6})
        \end{bmatrix}. 
    \end{split}
\end{equation*}
The target acceleration $\boldsymbol{\ddot{q}_t} \in {\mathbb{R}}^{6}$ will be presented in the next section, and the detailed expression of Lie derivatives $\mathcal{L}_{\boldsymbol{f}}h_{T,i}(\boldsymbol{x})$ and $\mathcal{L}_{\boldsymbol{g}}h_{T,i}(\boldsymbol{x})$ can be found in Appendix.

\subsubsection{Target Acceleration Calculation}

The target acceleration, $\boldsymbol{\ddot{q}_t} \in {\mathbb{R}}^6$, is calculated as follows:
\begin{equation}
    \begin{split}
        \ddot{q}_{t,i} =& -2k_{a,i}\delta_{v,i}\dot{q}_{d,i} - k^2_{a,i}(q_{t,i} - q_{d,i}), \ i = 1, \cdots, 6
    \end{split}
\end{equation}
where $k_{a,i}$ is a positive scalar and
\begin{equation} \label{eqn: damping ratio for the target acceleration}
    \delta_{v,i} = \delta_{v,m} + \tfrac{k_{\Delta p}|q_{d,i} - q_{t,i}|}{1 + k_{\Delta p}|q_{d,i} - q_{t,i}|}(\delta_{v,M} - \delta_{v,m})    
\end{equation}
with the minimum and maximum damping ratios, $\delta_{v,m}$ and $\delta_{v,M}$, and the proportional gain to the target-desired pose difference, $k_{\Delta p}$. (\ref{eqn: damping ratio for the target acceleration}) indicates that the magnitude of the desired twist can be reduced by increasing the damping ratio if the difference between the desired and target poses gets larger. It prevents excessive oscillation of $\boldsymbol{q_d}$ and $\boldsymbol{\dot{q}_d}$.

\section{Theoretical Analysis}

\begin{theorem} \label{thm: h_T is a multivariable CBF}
    Let $\boldsymbol{h_T}(\boldsymbol{x})$ denote $[h_{T,1}(\boldsymbol{x});$$\cdots;$$h_{T,1}(\boldsymbol{x})]$, then the safety set $\mathcal{C}_T$ is forward invariant with (\ref{eqn: rearranged dynamic model of the aerial manipulator}), (\ref{eqn: disturbance observer}), (\ref{eqn: DOB-based control law}), (\ref{eqn: unknown CBF estimation}) and (\ref{eqn: quadratic programming}).
\end{theorem}
\begin{proof}
     From (\ref{eqn: quadratic programming}), the inequality $\boldsymbol{A_T}\boldsymbol{\ddot{q}_d} + \boldsymbol{\sigma_T} \leq \boldsymbol{\hat{b}_T}$ is rearranged as $\boldsymbol{A_T}\boldsymbol{\ddot{q}_d} \leq \boldsymbol{b_T} - \boldsymbol{e_T} - \boldsymbol{\sigma_T}$ where $\boldsymbol{e_T} \triangleq [e_{T,1};\cdots;e_{T,6}]$ and
    \begin{equation*}
        \begin{split}
            \boldsymbol{b_T} \triangleq& \begin{bmatrix}
                \gamma_{T,1}(h(\boldsymbol{x})) + \mathcal{L}_{\boldsymbol{f}}h_{T,1}(\boldsymbol{x}) + \beta_{T,1}(\boldsymbol{x},\boldsymbol{\tilde{d}}) \\
                \vdots \\
                \gamma_{T,6}(h(\boldsymbol{x})) + \mathcal{L}_{\boldsymbol{f}}h_{T,6}(\boldsymbol{x}) + \beta_{T,6}(\boldsymbol{x},\boldsymbol{\tilde{d}})
            \end{bmatrix}. 
        \end{split}
    \end{equation*}
    Since $\boldsymbol{b_T} - \boldsymbol{A_T}\boldsymbol{\ddot{q}_d} = \boldsymbol{\dot{h}_T}(\boldsymbol{x},\boldsymbol{\dot{x}}) + \boldsymbol{\Gamma_{T}}\boldsymbol{h_T}(\boldsymbol{x})$ where $\boldsymbol{\Gamma_T} \triangleq \textrm{diag}\{[\gamma_{T,1};\cdots;\gamma_{T,6}]\}$, $\boldsymbol{A_T}\boldsymbol{\ddot{q}_d} \leq \boldsymbol{b_T} - \boldsymbol{e_T} - \boldsymbol{\sigma_T}$ is reformulated as $\boldsymbol{e_T} + \boldsymbol{\sigma_T} \leq \boldsymbol{\dot{h}_T}(\boldsymbol{x},\boldsymbol{\dot{x}}) + \boldsymbol{\Gamma_T}\boldsymbol{h_T}$. Because $-|e_{T,i}| \leq e_{T,i} \leq |e_{T,i}|$, (\ref{eqn: equation on e_T}) is arranged as $e_{T,i}(t) \geq (\tfrac{\beta_{T,h,i}}{k_{\beta,T,i}}-|e_{T,i}(t_0)|)e^{-k_{\beta,T,i}(t - t_0)} - \tfrac{\beta_{T,h,i}}{k_{\beta,T,i}}$. Then, from the proof of \cite[Theorem 1]{alan2022disturbance}, we obtain $e_{T,i}(t) + \sigma_{T,i} \geq (\textrm{max}\{|e_{T,i}(t_0)|,\tfrac{\beta_{T,h,i}}{k_{\beta,T,i}}\} - |e_{T,i}(t_0)|)e^{-k_{\beta,T,i}(t - t_0)} + \sigma_{T,i} - \textrm{max}\{|e_{T,i}(t_0)|,\tfrac{\beta_{T,h,i}}{k_{\beta,T,i}}\}$. Since $\sigma_{T,i} \geq \textrm{max}\{|e_{T,i}(t_0)|,\tfrac{\beta_{T,h,i}}{k_{\beta,T,i}}\}$, we prove that $e_{T,i}(t) + \sigma_{T,i} \geq 0$. Therefore, the following statement is true for all $\boldsymbol{x} \in \mathcal{D}$ and for all $t > t_0$:
    \begin{equation} \label{eqn: statement between inequalities}
        \boldsymbol{A_T}\boldsymbol{\ddot{q}_d} + \boldsymbol{\sigma_T} \leq \boldsymbol{\hat{b}_T} \ \rightarrow \boldsymbol{0_{6\times1}} \leq \boldsymbol{\dot{h}_T}(\boldsymbol{x},\boldsymbol{\dot{x}}) + \boldsymbol{\Gamma_T}\boldsymbol{h_T}.
    \end{equation}
    
    According to Appendix, $\mathcal{L}_{\boldsymbol{g}}h_{T,i}(\boldsymbol{x})$ is $-2(T_i(\boldsymbol{x}) - \tfrac{T_M + T_m}{2})\boldsymbol{E^{\top}_i}\boldsymbol{J_a}(\boldsymbol{\phi})\boldsymbol{K_d}$ where $\boldsymbol{J_a}(\boldsymbol{\phi}) = \boldsymbol{\Xi^{-1}}\boldsymbol{B^{-1}}(\boldsymbol{\phi})\boldsymbol{\hat{M}}(\boldsymbol{\phi})$. Thus, $\boldsymbol{A_T}$ in (\ref{eqn: quadratic programming}) is calculated as $-2\textrm{diag}\Big\{\Big[T_1(\boldsymbol{x})$$ - \tfrac{T_M + T_m}{2}$$;\cdots$$;T_6(\boldsymbol{x}) $$- \tfrac{T_M + T_m}{2}\Big]\Big\}$$\boldsymbol{J_a}(\boldsymbol{\phi})\boldsymbol{K_d}$. Since $\boldsymbol{\Xi}$, $\boldsymbol{B}(\boldsymbol{\phi})$, $\boldsymbol{\hat{M}}(\boldsymbol{\phi})$ and $\boldsymbol{K_d}$ are all invertible matrices for all $\boldsymbol{\phi} \in [-\pi, \pi] \times (-\tfrac{\pi}{2},\tfrac{\pi}{2}) \times [-\pi,\pi]$, $\boldsymbol{J_a}(\boldsymbol{\phi})\boldsymbol{K_d}$ is also invertible. 

    Let $\mathcal{N}_{T=0} \subset \{1,\cdots,6\}$ denote the set of indices that $T_i = \tfrac{T_M + T_m}{2}$, then we can analyze $\boldsymbol{A_T}\boldsymbol{\ddot{q}_d} + \boldsymbol{\sigma_T} \leq \boldsymbol{\hat{b}_T}$ row-by-row as follows:
    \begin{itemize}
        \item If $i \in \mathcal{N}_{T=0}$, then $\tfrac{\partial h_{T,i}}{\partial \boldsymbol{x}}$ becomes $\boldsymbol{0_{1 \times 36}}$ since  $\tfrac{\partial h_{T,i}}{\partial \boldsymbol{x}} = -2(T_i(\boldsymbol{x}) - \tfrac{T_m + T_M}{2})\boldsymbol{E_i^{\top}}\tfrac{\partial \boldsymbol{T}}{\partial \boldsymbol{x}}$. Thus, $\dot{h}_{T,i}(\boldsymbol{x},\boldsymbol{\dot{x}})$ becomes zero so that $0 \leq \dot{h}_{T,i}(\boldsymbol{x},\boldsymbol{\dot{x}}) + \gamma_{T,i}h_{T,i}(\boldsymbol{x})$ holds.
        
        \item Let $i_{f,1},\cdots,i_{f,n_T}$ denote the elements of $\mathcal{N}_{T \neq 0} \triangleq \{1,\cdots,6\} - \mathcal{N}_{T=0}$ where $n_T$ represents the number of elements in $\mathcal{N}_{T \neq 0}$, then the inequality $\boldsymbol{A_T}\boldsymbol{\ddot{q}_d} + \boldsymbol{\sigma_T} \leq \boldsymbol{\hat{b}_T}$ is rearranged as follows:
        \begin{equation}
          \boldsymbol{A_{T,\mathcal{N}_{T\neq0}}} \boldsymbol{\ddot{q}_d} \leq \boldsymbol{b_{T,\mathcal{N}_{T\neq0}}}
        \end{equation}
        where
        \begin{equation*}
            \begin{split}
                \boldsymbol{A_{T,\mathcal{N}_{T\neq0}}} \triangleq&  -2\begin{bmatrix}
                (T_{i_{f,1}}(\boldsymbol{x}) - \tfrac{T_m + T_M}{2})\boldsymbol{E^{\top}_{i_{f,1}}} \\
                \vdots \\
                (T_{i_{f,N_T}}(\boldsymbol{x}) - \tfrac{T_m + T_M}{2})\boldsymbol{E^{\top}_{i_{f,N_T}}}
            \end{bmatrix} \\
            &\times \boldsymbol{J_a}(\boldsymbol{\phi})\boldsymbol{K_d} \\
            \boldsymbol{b_{T,\mathcal{N}_{T\neq0}}} \triangleq& [b_{T,i_{f,1}} - \sigma_{T,i_{f,1}}; \cdots ; b_{T,i_{f,N_T}} - \sigma_{T,i_{f,N_T}}].
            \end{split}
        \end{equation*}
        Since $\boldsymbol{J_a}(\boldsymbol{\phi})\boldsymbol{K_d}$ is invertible, the rank of $\boldsymbol{A_{T,\mathcal{N}_{T\neq0}}}$ is $N_T$ unless $N_T = 0$. Thus, there exists $\boldsymbol{\ddot{q}_d^{\ast}} \in {\mathbb{R}}^6$ such that $\boldsymbol{A_{T,\mathcal{N}_{T\neq0}}}\boldsymbol{\ddot{q}_d^{\ast}} = \boldsymbol{b_{T,\mathcal{N}_{T\neq0}}} - \boldsymbol{b_{T,r}}$ where all the elements of $\boldsymbol{b_{T,r}} \in {\mathbb{R}}^{N_T}$ are positive scalars.
    \end{itemize}

    From the above discussion, we prove that there exists $\boldsymbol{\ddot{q}_{d,sol}} \in {\mathbb{R}}^6$ satisfying $\boldsymbol{A_T}\boldsymbol{\ddot{q}_{d,sol}} + \boldsymbol{\sigma_T} \leq \boldsymbol{\hat{b}_T}$ for all $(\boldsymbol{x},t) \in \mathcal{D} \times [t_0,\infty)$. Hence, from (\ref{eqn: statement between inequalities}), we prove that $\boldsymbol{0_{6\times1}} \leq \boldsymbol{\dot{h}_T}(\boldsymbol{x},\boldsymbol{\dot{x}}) + \boldsymbol{\Gamma_T}\boldsymbol{h_T}$ holds so that $\mathcal{C}_T$ is rendered forward invariance with (\ref{eqn: rearranged dynamic model of the aerial manipulator}), (\ref{eqn: disturbance observer}), (\ref{eqn: DOB-based control law}), (\ref{eqn: unknown CBF estimation}) and (\ref{eqn: quadratic programming}).
\end{proof}

\section{Experimental Results}

\begin{figure}[t]
\centering
\vspace{0.25cm}
\includegraphics[width = 0.48\textwidth]{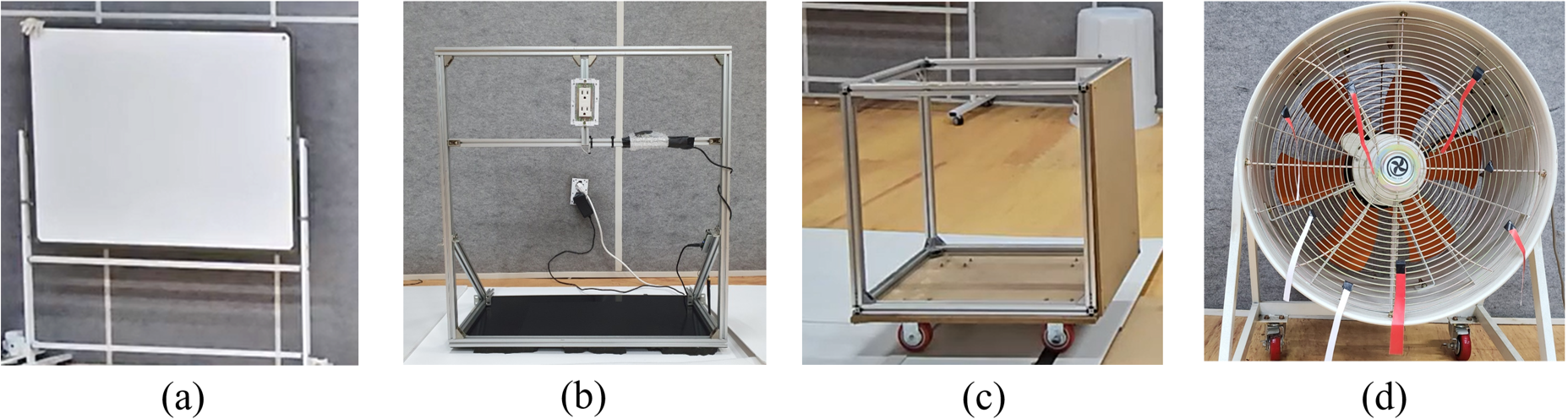}
\vspace{-0.40cm}
\caption{Equipment utilized in the actual experiments: (a) a vertical whiteboard, (b) a frame with a 110 V electric socket, (c)  a wheeled cart, and (d) an industrial fan.} \label{fig: experiment_environments}
\end{figure}

To validate the performance of the proposed controller in various types of APhI, we conduct actual experiments with the following four scenarios: 
\begin{enumerate}
    \item Pushing a static whiteboard (Fig. \ref{fig: thumbnail}(a)).
    \item Pulling a firmly attached plug (Fig. \ref{fig: thumbnail}(b)).
    \item Pushing a movable cart (Fig. \ref{fig: thumbnail}(c)).
    \item Pulling a plug out of a 110 V electric socket (Fig. \ref{fig: thumbnail}(b)).
\end{enumerate}
For scenarios 1) and 2), we compare the results from the proposed controller with two other existing controllers to show the enhanced control performance. For scenarios 3) and 4), we conduct repetitive experiments with the proposed controller to validate that our controller can repeatedly perform tasks involving sudden changes in dynamics. 

\subsection{Experimental Setups}

The aerial manipulator utilized for this work consists of two parts: a fully actuated hexacopter and a rigidly attached robotics arm. The fully actuated hexacopter which weighs 3.50 kg was assembled with the off-the-shelf frame DJI F550, six 9-inch APC LPB09045MR propellers, six KDE2314XF-965 motors with corresponding KDEXF-UAS35 electronic speed controllers (ESCs), six 3D printed thrust-tilting frames with $\alpha$ = $15 {}^{\circ}$, a 4S Turnigy Lipo battery to power up Intel NUC, a 6S Polytronics Lipo battery for the power supplement of the six ESCs, and Intel NUC for computing. On Intel NUC, Robot Operating System (ROS) noetic version is installed in Ubuntu 20.04, and the proposed control algorithm for the hexacopter and the navigation algorithm with Optitrack are executed. In the customized Pixhawk 4 connected to the Intel NUC, the rotational speeds of the six motors are controlled. 

For all the experiments, an industrial fan shown in Fig. \ref{fig: experiment_environments}(d) is used to generate a wind blast to imitate a more realistic situation, the parameter values are set as shown in Table \ref{table: common parameters}, and the target pose $\boldsymbol{q_t}$ is set from the human operator's laptop.


\begin{table}[h!]
\centering
\caption{Parameter values utilized in experiments.}
\vspace{-0.25cm}
\begin{tabular}{|c | c || c| c|} 
\hline
$\hat{m}$ & 3.50 & $T_m$ & 1.0 \\
\hline
$\hat{g}$ & 9.81 & $T_M$ & 15 \\
 \hline
 $\boldsymbol{\hat{J}}$ & {diag}\{[0.035;0.035;0.045]\} & $\gamma_{T,i = 1,\cdots,6}$ & 10 \\ 
 \hline
 $\boldsymbol{K_p}$ & {diag}\{[6.0,6.0,8.0,70,70,55]\} & $k_{\beta,T,i = 1,\cdots,6}$ & 10.1 \\
\hline
  $\boldsymbol{K_d}$ & {diag}\{[4.0,4.0,5.0,30,30,15]\} & $\sigma_{T,i = 1,\cdots,6}$ & 15 \\ 
 \hline
  $\boldsymbol{\Gamma_{\zeta}}$ & {diag}\{[1.0,1.0,1.0,0.10,0.10,0.50]\}  & $k_{a,i = 1,2,3,6}$ & 1.0 \\
 \hline
  $\boldsymbol{\Gamma_{\chi}}$ & {diag}\{[1.0,1.0,1.0,0.10,0.10,0.50]\}  & $k_{a,i = 1,2,3,6}$ & 5.0 \\
 \hline
 $\mu_{i=4,5}$ & 0.80 &  $\mu_{i=1,2,3,6}$  & 0.95 \\
 \hline 
\end{tabular}
\label{table: common parameters}
\end{table}

\subsection{Controller Comparison: Scenarios 1) and 2)}

For the scenarios 1) and 2), we compare the performance of the proposed method with two other baselines shown below:
\begin{itemize}
    \item First baseline: 
    Without a safety filter.  
    \item Second baseline: Direct adjustment of $\boldsymbol{\tau}$ based on $\boldsymbol{\tau_t}$ calculated as in (\ref{eqn: DOB-based control law}) instead of adjusting $\boldsymbol{q_d}$ and $\boldsymbol{\dot{q}_d}$.
    \item Proposed: The proposed safety filter.
\end{itemize}

\subsubsection{Pushing a Static Whiteboard}

For this scenario, we set $k_{\delta p}$, $\delta_{v,m}$ and $\delta_{v,M}$ as 0.5, 1.0 and 5.0, respectively. The human operator manually commanded the aerial manipulator to reach 0.3 m beyond the location of the whiteboard to exert pushing force on its surface.

\begin{figure}[t]
\centering
\vspace{0.25cm}
\includegraphics[width = 0.40\textwidth]{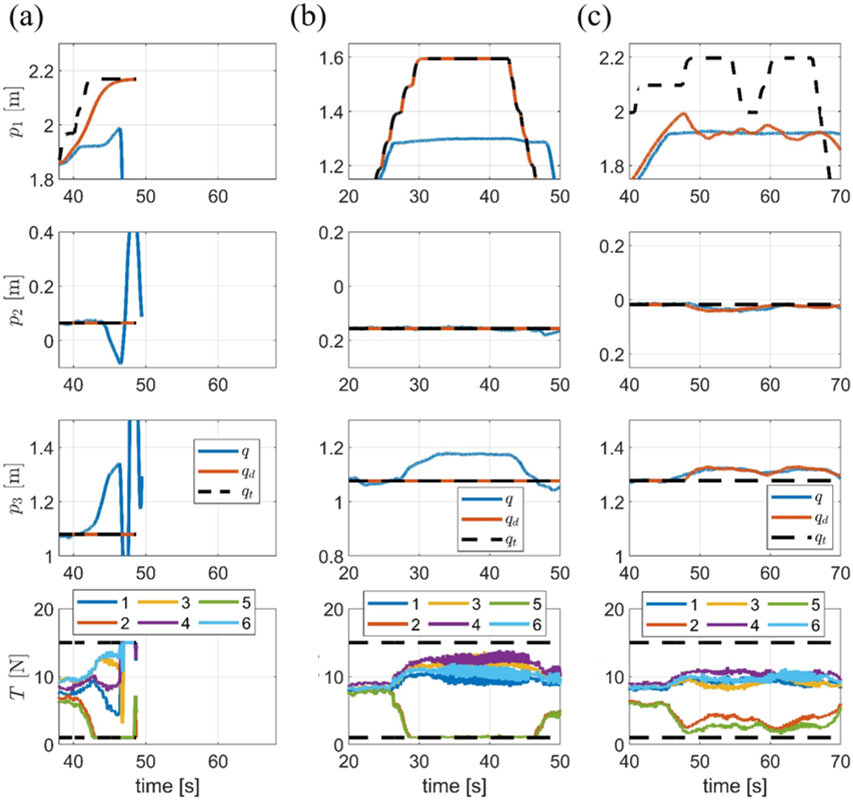}
\vspace{-0.30cm}
\caption{Histories of the vehicle's position ($\boldsymbol{p}$) and motor thrust values ($T_1,\cdots,T_6$) during the static whiteboard-pushing experiment with (a) the first baseline, (b) second baseline, and (c) the proposed method.} \label{fig: static pushing results}
\end{figure}

Fig. \ref{fig: static pushing results} presents the histories of the target, desired and actual position, and those of each motor's thrust. As observed in Fig. \ref{fig: static pushing results}(a), the vehicle failed to maintain its stability without any safety filter. Otherwise, in Figs. \ref{fig: static pushing results}(b) and \ref{fig: static pushing results}(c), we can notice that the vehicle successfully maintains its stability while pushing the whiteboard with the second baseline and proposed method. Even though there is no big difference between the results with the second baseline and the proposed method, we can note that $p_3$ shows slightly less overshoot with the proposed method.

\subsubsection{Pulling a Firmly Attached Plug}

For this scenario, we set $k_{\delta p}$, $\delta_{v,m}$ and $\delta_{v,M}$ as 5.0, 1.0, and 5.0, respectively. To make a pulling movement, the operator set the target pose to 0.2 m away from the socket.

\begin{figure}[t]
\centering
\vspace{0.25cm}
\includegraphics[width = 0.40\textwidth]{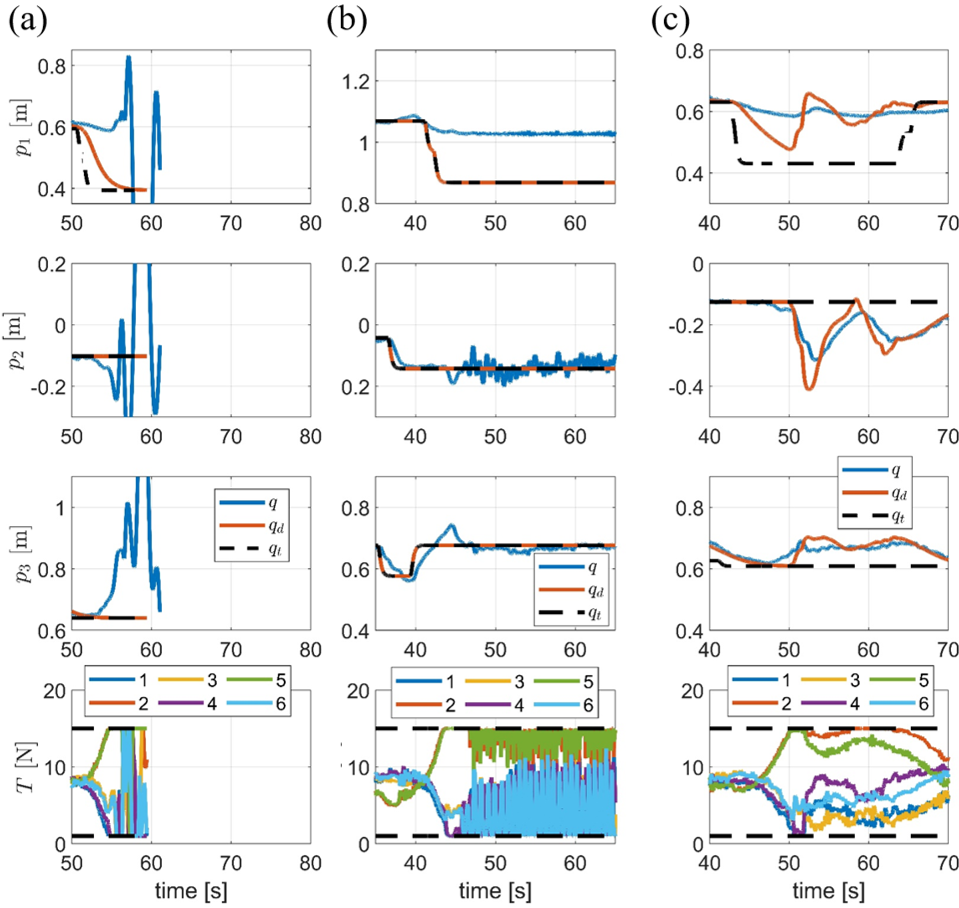}
\vspace{-0.30cm}
\caption{Histories of the vehicle's position ($\boldsymbol{p}$) and motor thrust values ($T_1,\cdots,T_6$) while pulling a firmly attached plug with (a) the first baseline, (b) second baseline, and (c) the proposed method.} \label{fig: static pulling results}
\end{figure}

Likewise in Fig. \ref{fig: static pushing results}(a), from Fig. \ref{fig: static pulling results}(a), we can observe that the aerial manipulator showed an oscillating behavior without any safety filter while pulling the plug firmly attached to the socket. On the other hand, Figs. \ref{fig: static pulling results}(b) and \ref{fig: static pulling results}(c) indicate that the second baseline and the proposed method lead to the vehicle's safe pulling. However, from the plot of $\boldsymbol{T}$ in Fig. \ref{fig: static pulling results}(b), we can note that there occurred a large oscillation in motor thrust. In the attached video, we can observe the large oscillation of roll and pitch angles due to the oscillating behavior of the motor thrust. 

\subsection{Validation on Tasks Involving Sudden Changes in Dynamics: : Scenarios 3) and 4)}

Through scenarios 3) and 4), we aim to emphasize the proposed method's repeatability in situations involving sudden changes in dynamics such as the sudden disappearance of the interaction wrench. Hence, we conducted five repetitive experiments for each task.

\subsubsection{Pushing a Movable Cart}

For this scenario, $k_{\delta p}$, $\delta_{v,m}$ and $\delta_{v,M}$ are set as 0.5, 1.0, and 5.0, respectively. We commanded the aerial manipulator to consistently push the cart (Fig. \ref{fig: experiment_environments}(c)) until its rear wheels go over the left black line shown in Fig. \ref{fig: thumbnail}(c).

\begin{figure}[t]
\centering
\includegraphics[width = 0.29\textwidth]{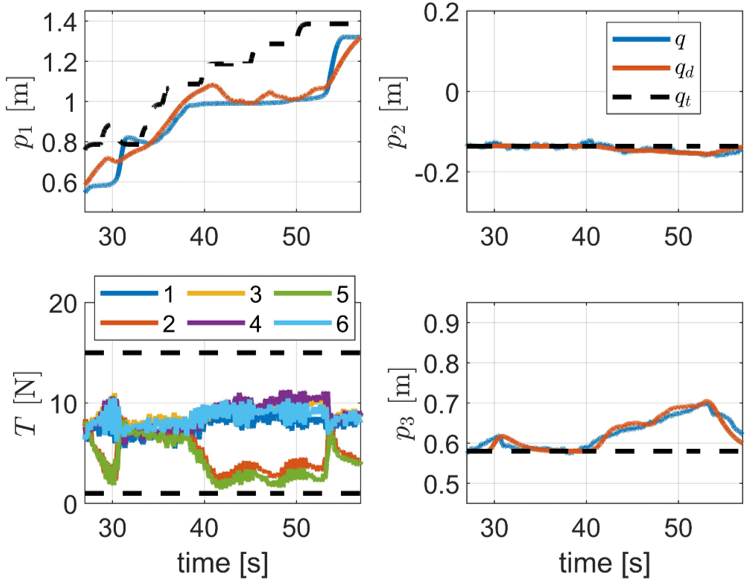}
\vspace{-0.35cm}
\caption{Histories of the vehicle's position ($\boldsymbol{p}$) and motor thrust values ($T_1,\cdots,T_6$) of the first cart-pushing experiment.} \label{fig: dynamic pushing result}
\end{figure}

Fig. \ref{fig: dynamic pushing result} reports the histories of the vehicle's position and motor thrust while conducting the first one among five cart-pushing experiments. As seen in the position plots, $\boldsymbol{p_{d}}$ actively modifies its value accordingly with the cart's movement.

\subsubsection{Pulling a Plug out of the Socket}

For this scenario, we put 5.0, 1.0, and 5.0 into $k_{\delta p}$, $\delta_{v,m}$ and $\delta_{v,M}$, respectively, and set $p_{t.1} = q_{t,1}$ as 0.2 m away from the grabbing position.

\begin{figure}[t]
\centering
\vspace{0.25cm}
\includegraphics[width = 0.29\textwidth]{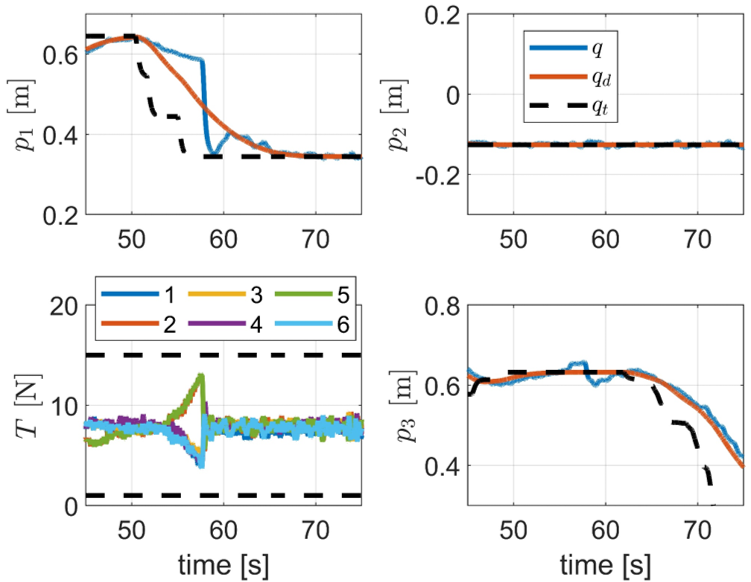}
\vspace{-0.35cm}
\caption{Histories of the vehicle's position ($\boldsymbol{p}$) and motor thrust values ($T_1,\cdots,T_6$) of the fourth plug-pulling experiment.} \label{fig: plug pulling result}
\end{figure}

In Fig. \ref{fig: plug pulling result}, the position and thrust histories of the fourth one among five plug-pulling experiments. In the left upper plot of Fig. \ref{fig: plug pulling result}, the history of $q_1$ shows that the plug is extracted around 57 seconds and the vehicle successfully stabilizes right after the plug separation.

\section{Conclusions}

This paper presented a safety filter based on a disturbance observer (DOB)-based control structure, designed to enhance pose-tracking performance while ensuring flight safety during physical interaction. First, we derived the Lagrange-Euler model for an aerial manipulator consisting of a fully actuated hexacopter and a rigidly attached robotic arm. Using this model, we developed a safety filter that accounted for motor thrust limits, external disturbances, and model uncertainties, integrated with the DOB-based control method.
We also proved that the proposed safety filter and controller ensured forward invariance of the safety set with respect to motor thrust limits. To validate the  performance of the proposed method  over existing control strategies for aerial physical interaction (APhI), we conducted comparative experiments. These included pushing a static structure and pulling a firmly attached object using the aerial manipulator.
Additionally, we performed repeated tests on tasks with sudden dynamic changes, such as pushing a movable cart and extracting a plug from a socket, to verify the controller’s ability to handle these dynamic shifts. Future work may involve designing a hybrid controller to better manage abrupt changes in dynamic models or extending the system for haptic-based teleoperation of an aerial manipulator interacting with its environment.

\vspace{-0.20cm}

\section*{Appendix}

The detailed expressions of the Lie derivatives are calculated as follows:
\begin{equation} \label{eqn: initial models of the Lie derivatives}
    \begin{split}
        \mathcal{L}_{\boldsymbol{f}}h_{T,i}(\boldsymbol{x}) =& -2(T_i(\boldsymbol{x}) - \tfrac{T_m + T_M}{2})\boldsymbol{E^{\top}_i}\tfrac{\partial \boldsymbol{T}}{\partial \boldsymbol{x}}(\boldsymbol{x})\boldsymbol{f}(\boldsymbol{x}), \\ \mathcal{L}_{\boldsymbol{g}}h_{T,i}(\boldsymbol{x}) =& -2(T_i(\boldsymbol{x}) - \tfrac{T_M + T_m}{2})\boldsymbol{E^{\top}_i}\boldsymbol{J_a}(\boldsymbol{\phi})\boldsymbol{K_d}
    \end{split}
\end{equation}
where
\begin{equation*} 
    \begin{split}
        \tfrac{\partial \boldsymbol{T}}{\partial \boldsymbol{x}}(\boldsymbol{x}) =& \Big[\tfrac{\partial \boldsymbol{T}}{\partial \boldsymbol{q}}(\boldsymbol{x}) \vdots -\boldsymbol{J_a}(\boldsymbol{\phi})(\boldsymbol{K_d} + \boldsymbol{\mu^{-1}}\boldsymbol{\Gamma_{\zeta}}) \vdots \\
        & \boldsymbol{J_a}(\boldsymbol{\phi})\boldsymbol{\mu^{-1}}\boldsymbol{\Gamma_{\zeta}} \vdots \boldsymbol{J_a}(\boldsymbol{\phi}) \vdots \boldsymbol{J_a}(\boldsymbol{\phi})\boldsymbol{K_p} \vdots \boldsymbol{J_a}(\boldsymbol{\phi})\boldsymbol{K_d} \Big] \\
        \tfrac{\partial \boldsymbol{T}}{\partial \boldsymbol{q}}(\boldsymbol{x}) =& -\boldsymbol{J_a}(\boldsymbol{\phi})\boldsymbol{K_p} - \Big[\boldsymbol{0_{3\times 3}} \vdots \tfrac{\partial \boldsymbol{J_a}(\boldsymbol{\phi})}{\partial \phi_1}\boldsymbol{J^{-1}_a}(\boldsymbol{\phi})\boldsymbol{T}(\boldsymbol{x}) \vdots \\ & \quad \ \tfrac{\partial \boldsymbol{J_a}(\boldsymbol{\phi})}{\partial \phi_2}\boldsymbol{J^{-1}_a}(\boldsymbol{\phi})\boldsymbol{T}(\boldsymbol{x}) \vdots \tfrac{\partial \boldsymbol{J_a}(\boldsymbol{\phi})}{\partial \phi_3}\boldsymbol{J^{-1}_a}(\boldsymbol{\phi})\boldsymbol{T}(\boldsymbol{x}) \Big].
    \end{split}
\end{equation*}
with $\boldsymbol{J_a}(\boldsymbol{\phi}) \triangleq \boldsymbol{\Xi^{-1}}\boldsymbol{B^{-1}}(\boldsymbol{\phi})\boldsymbol{\hat{M}}(\boldsymbol{\phi})$.

\section*{Acknowledgement}

\vspace{-0.05cm}

The authors would like to thank Dohyun Eom (Seoul National University) for his help conducting the experiments.

\normalem
\bibliographystyle{IEEEtran}
\bibliography{myreference}

\end{document}